\documentclass[11pt]{article}
\usepackage[a4paper,top=3cm,bottom=2cm,left=3cm,right=3cm,marginparwidth=1cm]{geometry}
\usepackage{amsmath}
\usepackage{bm}
\usepackage{amssymb}
\usepackage{mathtools}
\usepackage{amsthm}
\usepackage{algpseudocode}
\usepackage{algorithm}
\usepackage{subfigure}
\usepackage{wrapfig}
\usepackage{stackengine}
\usepackage{color}
\usepackage{times}

\pdfoutput=1

\usepackage[utf8]{inputenc}

\newtheorem{thm}{Theorem}
\newtheorem{defi}{Definition}

\newtheorem{lemma}{Lemma}
\newtheorem{corollary}{Corollary}

\newtheorem{obser}{Observation}
\newtheorem{claim}{Claim}

\theoremstyle{definition}

\newtheorem{remark}{Remark}
\newtheorem{prop}{Proposition}

\usepackage[utf8]{inputenc}

\title{Parametrized Accelerated Methods,\\ Free of Condition Number}
\author{
   Chaoyue Liu,\ Mikhail Belkin \\
  Department of Computer Science and Engineering\\
  The Ohio State University\\
  \texttt{liu.2656@osu.edu, mbelkin@cse.ohio-state.edu}\\
}

\date{\today}
\begin{document}

\maketitle

\begin{abstract}
Analyses of accelerated (momentum-based) gradient descent usually assume bounded condition number to obtain exponential convergence rates. However, in many real problems, e.g., kernel methods or deep neural networks, the condition number, even locally, can be unbounded, unknown or mis-estimated.  This poses problems in both implementing and analyzing accelerated  algorithms. In this paper, we address this issue by proposing parametrized  accelerated methods by considering the condition number as a free parameter. We provide spectral-level analysis for several important accelerated algorithms, obtain explicit expressions and improve  worst case convergence rates. Moreover, we show that those algorithm converge exponentially even when the condition number is unknown or mis-estimated.

\end{abstract}

\section{Introduction}
Accelerated (momentum-based) gradient descent and its variants are arguably among the most popular optimization methods in modern machine learning. It is  a workhorse of optimization for deep neural networks and achieves state-of-the-art results in a range of applications~\cite{bengio2013advances,gregor2015draw,xu2015show}.

Momentum-based algorithms are a class of first order iterative methods which use gradient evaluations from several previous iterations. These methods can be shown to reduce the number of iterations compared to ordinary gradient descent. There is an extensive literature on analyzing such accelerated schemes, notably Nesterov's accelerated gradient descent (Nesterov's AGD) (see, e.g.,~\cite{nesterov2013introductory, bubeck2015convex}, and references therein). 

We note that most analyses of momentum-based accelerated methods assume strong convexity (bounded condition number $\kappa$) to obtain exponential\footnote{Called {\it linear} in the optimization literature.} convergence rates, i.e., $O(e^{-k/\sqrt{\kappa}})$, where $k$ is the number of iterations. Only much slower rates $O(1/k^2)$ can be derived without that assumption~\cite{bubeck2015convex}.  
Moreover the optimal choice of parameters for these accelerated methods depends explicitly on the condition number $\kappa$.

However, in many real problems, $\kappa$ can be very large or even unbounded.
For example, it can be shown that for smooth kernels $\kappa$ grows nearly exponentially with the number of data points~\cite{wendland2004scattered,2018arXiv180103437B}. While neural networks are generally non-convex, Hessian matrices at minima appear to have many small eigenvalues resulting in high (local) condition  numbers~\cite{2016arXiv161107476S}.
The condition number is generally difficult to estimate. Such estimation is costly (potentially as expensive as full matrix inversion) and numerically unstable, requiring estimating the inverse of the smallest eigenvalue of a positive definite Hessian matrix.
When $\kappa$ is not known or mis-estimated, we generally have no guarantee for the validity of momentum-based methods. 
Moreover, if the condition number is known but very large, the exponential theoretical rate $O(e^{-k/\sqrt{\kappa}})$ can still be very slow, and potentially requires more computation  than the Newton's method.

In this paper, our primary goal is to understand performance of momentum-based algorithms and their parameter selection, when the condition number is very large or unknown. To that end, whenever parameter choice for specific algorithms depends on knowing the condition number, we propose to parametrize the algorithms by treating those parameters as ``free''. These parametrized algorithms should be proven to converge, for all choices of the parameter, in order to be validated.

To be able to do that and to simplify the analysis, we consider quadratic objective functions. This is an important  case which allows for much precise analysis than general convex functions.
Moreover, even when the objective function is non-convex but smooth, as is the case for many neural networks, 
it can be approximated by a quadratic function near any of its local minima.

Most previous analyses focus on the worst case convergence behavior of such momentum-based algorithms. 
However, convergence can be much faster depending on the specific spectral properties of the Hessian.
 Thus we expect spectral level analysis, providing rates for each individual eigenvalue, to be much more precise. Still, to the best of our knowledge, explicit spectral level representations for these algorithms (except for the Chebyshev semi-iterative method) are not found in the literature. 
A recent paper \cite{neubauer2017nesterov} explored spectral-level properties of Nesterov's AGD, but still did not give an  explicit expression for the spectral-level convergence rate.

In this paper, we study and provide explicit spectral analysis for three important momentum-based accelerated methods: Nesterov's AGD, Chebyshev semi-iterative method (Chebyshev) and Second-order Richardson method (SOR). Nesterov's AGD is very commonly used in practice and extensively analyzed~\cite{bubeck2015geometric,arjevani2015lower,su2014differential}. The classical Chebyshev semi-iterative method~\cite{lanczos1952solution,golub1961chebyshev} has a number of optimality properties, while SOR~\cite{frankel1950convergence,riley1954iteration}, also known as the heavy ball method~\cite{polyak1964some}, is  the simplest fixed coefficient momentum scheme. 
In this paper, we collectively call the set of these three methods the {\it accelerated class}. \\

\noindent{\bf Our Contribution.}
\begin{itemize}
\item In this work we give explicit spectral level representations  for the accelerated class methods. As far as we know, these are the first explicit expressions for Nesterov's AGD and SOR methods. 
We analyze and compare their convergence rates.
In particular, we show that these algorithms converge exponentially for each eigenvalue, which improves the  rate obtained in~\cite{neubauer2017nesterov}. We also express their worst-case convergence guarantees in terms of what we call Chebyshev numbers, which can be computed explicitly. Interestingly, we observe that Nesterov's AGD has the slowest worst case convergence rate among the accelerated class, and that none of the algorithms accelerate all scenarios.

 \item We show all of the accelerated class algorithms converge even when the condition number is mis-specified. We also see how their rates of convergence depend on  the choice of the parameter, corresponding to the condition number. We also provide a comparison of these methods in the ``beyond the condition number'' non-strictly convex regime. Additionally we show that in that regime all of the accelerated class methods converge faster than ordinary gradient descent. 
\end{itemize}

\noindent{\bf Organization. }The paper is organized as follow: In section~\ref{sec:preliminary}, we list some useful preliminaries and notations. In section~\ref{sec:algorithms}, we briefly introduce the accelerated class algorithms. In section~\ref{sec:representation}, we provide the explicit expressions of spectral-level convergence rate, and show basic but important observations. In section~\ref{sec:strongly}, we analyze the convergence behavior of the accelerated class methods, in the strongly convex setting. In section~\ref{sec:parameterize}, we propose parametrized methods which do not require any assumption on or knowledge of condition number, and provide analysis and compare their convergence performance. We also discuss the effect of changing the acceleration parameter. Proofs of main theorems can be found in the Appendix.

\section{Preliminaries and Notations}\label{sec:preliminary}
In this section, we introduce notation and some important background definitions and results, see, e.g.~\cite{achieser2013theory,bubeck2015convex,golub1961chebyshev}.

Consider the problem of minimizing a least square objective function:
\begin{equation}
    f(w) = \frac{1}{2}\|y-Xw\|^2, \ w\in \mathcal{W},
    \label{eq:objfunc}
\end{equation}
where $\mathcal{W}$ is a Hilbert space, $X$ is a linear operator from $\mathcal{W}$ to another Hilbert space $\mathcal{Y}$ and $y\in \mathcal{Y}$. $w$ is usually interpreted as \emph{weight} in some literature.

The Hessian operator $X^*X$ is positive definite, and hence $f(w)$ is convex. The gradient of $f$ is 
\begin{equation}
    \nabla f(w) = X^*y - X^*Xw.
    \label{eq:gradient}
\end{equation}
If $f(w)$ is further strictly convex, there would be a unique optimizer $w^*=(X^*X)^{-1}X^*y$.

\begin{defi}[Strong Convexity]
function $f: \mathcal{W}\to \mathbb{R}$ is $\alpha$-strongly convex if it satisfies
\begin{equation}
    f(w) \ge f(v) + \langle \nabla f(v), w-v \rangle + \frac{\alpha}{2}\|w-v \|^2, \ \forall w,v\in \mathcal{W}.
\end{equation}
\end{defi}
\begin{defi}[Smoothness]
function $f: \mathcal{W}\to \mathbb{R}$ is $\beta$-smooth if it satisfies
\begin{equation}
    |f(w) - f(v) - \langle \nabla f(v), w-v \rangle| \le  \frac{\beta}{2}\|w-v \|^2, \ \forall w,v\in \mathcal{W}.
\end{equation}
\end{defi}
When both $\alpha$-strong convexity and $\beta$-smoothness are satisfied, one can define the \emph{condition number} $\kappa:= \beta/\alpha$.

\begin{defi}[{Operator Spectrum}]
Let $\mathcal{H}$ be a Hilbert space. Given an operator $T: \mathcal{H}\to \mathcal{H}$, the (operator) spectrum of $T$ is
\begin{equation}
    \textnormal{sp}(T) = \{\mu\in \mathbb{C}| T-\mu I \textrm{ is not invertible}\}.
\end{equation}
Every $\mu\in\textnormal{sp}(T)$ is called an eigenvalue of $T$.
\end{defi}
\begin{prop}
When $T$ is self-adjoint, $sp(T)\subset \mathbb{R}$. If $T$ is further positive definite, $sp(T) \subset [0,\infty)$.\label{prop:sp}
\end{prop}

\begin{defi}[Chebyshev Polynomials] The $k$-th Chebyshev polynomial (of the first kind), denoted as $C_k$, is defined as 
\begin{eqnarray}\label{eq:chebpolynomial}
    C_k(x) = \left\{\begin{array}{ll}
    \cos (k \cos^{-1} x), & \textnormal{if } |x|\le 1,\\
    \cosh(k \cosh^{-1} x), & \textnormal{if } x>1,\\
    (-1)^k\cosh(k\cosh^{-1} (-x)), & \textnormal{if } x< -1.
    \end{array}
    \right.
\end{eqnarray}
\end{defi}
\begin{remark}
Note that $C_k$ is a polynomial of degree $k$.
\end{remark}
\begin{prop}
\textnormal{Chebyshev polynomials satisfy the following recursive relations}
\begin{equation}
  C_{k+1}(x) = 2xC_k(x)-C_{k-1}(x), \quad \forall k\ge 1.
  \label{eq:chebrecur}
\end{equation}
\end{prop}
 \begin{thm}
 Let $\Pi_k$ is the set of all polynomials of degree $k$ with leading coefficient 1. Then
 \begin{equation}
   \min_{P_k\in\Pi_k}  \max_{x\in[-1,1]}|P_k(x)|
 \end{equation}
 has a unique optimizer $P_k^*=\frac{1}{2^{k-1}}C_k$.
 \label{thm:chebpoly}
 \end{thm}

\section{Momentum-based Accelerated Methods}\label{sec:algorithms}
\label{sec:2}

In this section, we introduce a few classical momentum schemes, under $\alpha$-strong convexity and $\beta$-smoothness conditions. For the moment, we also assume the condition number $\kappa$ is known. 

Aiming at optimizing $f$, defined in Eq.(\ref{eq:objfunc}), first-order iterative methods (e.g., Gradient descent) utilize first-order (gradient) information of the objective function $f$ to iteratively approximate optimizer $w^*$ by an approximator $w_{k}$. 

Define error $\xi_{k} = w_{k}-w^*$. Then the norm $\|\xi_{k}\|$ indicates how far we are away from the optimizer in the current iteration. Moreover, the excess risk can be expressed as 
\begin{equation}
   f(w_{k})-f(w^*) = \|X\xi_{k}\|^2.
   \label{eq:excessrisk}
\end{equation}
Define operator $B:=I-\eta X^*X$, where $I$ is the identity operator and $\eta$ is a scalar, called \emph{step size}, to be chosen.
In this paper, we always set $\eta = 1/\beta$, to avoid potential over-shooting issues. In addition, we introduce parameter $\rho = 1-1/\kappa \in [0,1)$. 

It is easy to see that $B$ is self-adjoint. Since Hessian $X^*X$ is positive definite, by Proposition~\ref{prop:sp} and the $\beta$-smoothness condition, 
sp$(B)\subseteq[0,\rho]$. It is also important to note that an eigen-space of $B$ is also an eigen-space of Hessian, since $B$ commutes with the Hessian $X^*X$, with an eigenvalue correspondence:
\begin{equation}
    \mu \longleftrightarrow 1-\mu_H/\beta,
\end{equation}
where $\mu_H$ is the corresponding Hessian-eigenvalue.
 
By spectral mapping theorem, the spectrum of any polynomial $P(B)$ in $B$ satisfies
 \begin{equation}
     \textnormal{sp}(P(B)) = P(\textnormal{sp}(B)) \subseteq P([0,\rho]).
 \end{equation}

\noindent{\bf Gradient Descent.} 
The (plain) gradient descent algorithm, uses full gradient information, Eq.(\ref{eq:gradient}), to iteratively update the approximator $w_{k}$, following the rule:
\begin{equation}
   w_{k+1} = w_{k} -\eta \nabla f(w_{k}), \quad \textrm{with } k\in \mathbb{N},
\end{equation}

It is not hard to see that the error $\xi_{k}$ satisfies
\begin{equation}
    \xi_{k} = B^k \xi_{0}, \quad k\in \mathbb{N}.
    \label{eq:power}
\end{equation}
\begin{remark} $B^k$ is a power of $B$, hence a polynomial of degree $k$ in $B$. In this paper, we also call the plain gradient descent as {\it power method} (also known as (first-order) Richardson or Landweber method in the literature).
\end{remark}

It is well known that~\cite{bubeck2015convex}, for $\alpha$-strongly convex and $\beta$-smooth objective functions, power method needs $O(\kappa\log(1/\epsilon))$ iterations to achieve an excess risk of $\epsilon$, while the theoretical worst case lower bound is proven to be $\Omega(\sqrt{\kappa}\log(1/\epsilon))$. 

Then, it is natural to ask: with a gradient oracle, can we design a practical algorithm which uses only $O(\sqrt{\kappa}\log(1/\epsilon))$ iterations, such that it converges faster? Or from another point of view, how can one make the excess risk $f(w_{k})-f(w^*)$ as small as possible, for every $k\in \mathbb{N}$? 

\subsection{Acceleration Problem}
To formulate the above questions, we consider a sequence $\{P_k\}_{k\in\mathbb{N}}$ of real-valued polynomials, where subscript $k$ indicates the degree of the polynomials, and let
\begin{equation}
    \xi_{k} = P_k(B) \xi_{0}, \quad k\in \mathbb{N}.
    \label{eq:generalrelation}
\end{equation}
We aim at finding a ``best" choice of the sequence $\{P_k\}_{k\in\mathbb{N}}$, such that the excess risk $\|X\xi_k\|^2$ is minimized for each k.
\begin{remark}
The reason to consider such polynomials is that polynomials are a much richer space than monomials, as in Eq.(\ref{eq:power}), but the time complexity remains of the same order. However the memory requirements generally grow linearly with the degree of the polynomials. This can be addressed by considering polynomial families with short recurrence relations, e.g., Chebyshev polynomials. 
\end{remark}

Note that the excess risk $\|X\xi_k\|^2$ depends on the matrix $X$. Hence the optimal optimization method is dependent on the properties of the data, making different solutions optimal for different optimization problems.  As  we will see in Section~\ref{sec:representation}, there exists no universal algorithm which is optimal for all optimization problems.


Hence, we first study a related data-independent problem. We introduce the following convenient quantity:
\begin{defi}[Chebyshev Number]
Given a real number $\rho\in (0,1)$, we define the Chebyshev number of polynomial $P$, which satisfies the condition $P(1) = 1$, as
\begin{equation}
    Ch_{\rho}(P) = \max_{\lambda\in[0,\rho]}|P(\lambda)|.
    \label{eq:chebnumb}
\end{equation}
\end{defi}
As will see below, the Chebyshev number measures the worst case convergence rate, which is data independent.\\

\noindent{\bf Optimization Problem.} Given $\rho\in(0,1)$, find a sequence of polynomials $\{P_k^*\}_{k\in\mathbb{N}}$ such that
\begin{equation}
P_k^* = \arg\min_{P_k} Ch_{\rho}(P_k), \quad \textnormal{subject to } P_k(1) = 1.
\label{eq:optiprob}
\end{equation}

\begin{remark} The extra condition  $P_k(1)=1$ is necessary since when $B=I$ the gradient is zero and first-order algorithms can not update $w_{k}$ at all.
\end{remark}

This optimization problem can be viewed as minimizing the   excess risk $f(w_{k})-f(w^*)$ in the worst case scenario,  assuming $\alpha$-strong convexity and $\beta$-smoothness on $f$. This is formalized by the following theorem:
\begin{prop}
\begin{equation}
  f(w_{k})-f(w^*) \le Ch^2(P_k)(f(w_{0})-f(w^*)),
\end{equation}
and the equality holds when $sp(B) = \{\lambda_{Ch}\}$, where $\lambda_{Ch}$ is the maximizer of the corresponding Chebyshev number, in Eq.(\ref{eq:chebnumb}).
\end{prop}
Since the inequality can be an equality, $Ch^2_{\rho}(P_k)$ is the worst case convergence rate, correspondingly, the case when $sp(B) = \{\lambda_{Ch}\}$ gives the slowest convergence.

\subsection{Accelerated Class Methods}
All the following algorithms explicitly rely on the assumptions of strong convexity and smoothness, and are proven~\cite{bubeck2015convex} to use $O(\sqrt{\kappa}\log(1/\epsilon))$ iterations to reach a excess risk of $\epsilon$.\\

{\noindent}{\bf Chebyshev (Semi-Iterative) Method.} 
By theorem~\ref{thm:chebpoly}, the solution $P_k^*$ of the optimization problem, Eq.(\ref{eq:optiprob}), can be shown to be unique and has the form of "normalized" Chebyshev polynomials\footnote{ Strictly speaking, $P_k^*(x) = C_k(\frac{2x-\rho}{\rho})/C_k(\frac{2-\rho}{\rho})$, which is nothing else but just a scaled version of $P_k^*(x)$.}~\cite{golub1961chebyshev,flanders1950numerical}:
\begin{equation}
  P_k^*(x) = \frac{C_k(x/\rho)}{C_k(1/\rho)}, \quad k\in \mathbb{N},
  \label{eq:chebyshev}
\end{equation}

Combined with Eq.(\ref{eq:generalrelation}), the recursive relation, Eq.(\ref{eq:chebrecur}), of Chebyshev polynomials
 allows us to compute $P_k^*(B)$, hence $\xi_k$, recursively without storing earlier information except $\xi_{k-1}$ and $\xi_{k-2}$. Thus, it is efficient in both time and space. The induced update rule for weight $w_k$ is
 \begin{subequations}
\begin{eqnarray}
& &w_{k+1} = w_k - \gamma_{k+1}\eta \nabla f(w_k) \nonumber\\
& &\quad \quad \quad + (\gamma_{k+1}-1)(w_k-w_{k-1}), \ k\in \mathbb{N};\\
& &w_{-1}=0.
\end{eqnarray}
 \end{subequations}
with the coefficients $\gamma_k$ determined by
\begin{subequations}
\begin{eqnarray}
  & &\gamma_{k+1} = 1/(1-\rho^2\gamma_k/4), \quad \textrm{for } k\ge 2;\\
  & &\gamma_1 = 1,\quad \gamma_2 = 2/(2-\rho^2).
\end{eqnarray}
 \end{subequations}
{\noindent}{\bf Second-order Richardson Iterative Method.} SOR updates $w_k$ following the rule: 
\begin{eqnarray} \label{eq:2ndrich}
   w_{k+1} = w_k - c_1\eta\nabla f(w_k) + c_2(w_k-w_{k-1}), \  k\ge 1;\nonumber\\
   w_1 = w_0 - \eta\nabla f(w_0).\quad\quad\quad\quad\quad\quad\quad\quad
\end{eqnarray}

where $c_1,c_2$ are time-independent coefficients. The displacement between the last two history records $w_k-w_{k-1}$ is usually interpreted as {\it momentum}.

The analysis of Frankel and Young~\cite{frankel1950convergence,young1954iterative} suggests the following coefficients choice\footnote{\ For the case of quadratic objective functions.},
\begin{equation}
    c_1 = \frac{2}{1+\sqrt{1-\rho^2}}:=\gamma, \quad c_2 = \gamma -1.
    \label{eq:2orcoff}
\end{equation}
$w_{k}$'s in Eq.(\ref{eq:2ndrich})  satisfy a recurrence relation:
\begin{equation}
    \xi_{k+1} = \gamma B \xi_k + (1-\gamma){\xi}_{k-1}, \ k\ge 1;\quad \xi_1 = B \xi_0.
    \label{eq:richardsonrec}
\end{equation}
It is important to note that the coefficient $\gamma_k$ in Chebyshev method is time changing and $\lim_{k\to\infty}\gamma_k =\gamma$, therefore, SOR can be viewed as the limiting case of the Chebyshev method.\\

{\noindent}{\bf Nesterov's Accelerated Gradient Descent.} Introduced by Nesterov in 1983~\cite{nesterov1983method}, Nesterov's AGD iteratively updates the approximator as follows\footnote{\ This is the constant parameter scheme.}~\cite{nesterov2013introductory}:
 \begin{subequations}
\begin{eqnarray}
    w_{k+1}=u_k-\frac{1}{\beta}\nabla f(u_k),\quad\quad\  \ \quad\quad\quad\quad \quad\ \\
    u_{k+1}=\left(1+\frac{\sqrt{\kappa}-1}{\sqrt{\kappa}+1}\right)w_{k+1}-\frac{\sqrt{\kappa}-1}{\sqrt{\kappa}+1}w_k.
    \end{eqnarray}\label{eq:nesterov}
\end{subequations}
Defining $\gamma'=1+(\sqrt{\kappa}-1)/(\sqrt{\kappa}+1)$, 
one can find the recurrence relation:
\begin{equation}
  {\xi}_{k+1} = \gamma' B {\xi}_k + (1-\gamma')B{\xi}_{k-1}, \ k\ge 1.   \label{eq:nesterovrec}
\end{equation}
\begin{remark}
Basic algebraic computation shows that $\gamma' = 2/(1+\sqrt{1-\rho})$. When comparing to the definition of $\gamma$ in Eq.(\ref{eq:2orcoff}), one note that the only difference is the absence of square  on $\rho$. This feature is essential leading to different performance than the other two accelerated methods, as will be seen in Section \ref{sec:representation} and \ref{sec:strongly}.
\end{remark}
By induction, one can easily show that recurrence relations, Eq.(\ref{eq:richardsonrec}) and (\ref{eq:nesterovrec}), also imply polynomial-type relations as in Eq.(\ref{eq:generalrelation}). We call the corresponding polynomials $R_k$ and $N_k$, respectively.

In the rest of this paper, we call the collection of Chebyshev, SOR and Nesterov's AGD as the \emph{accelerated class} methods/algorithms.

\section{Spectral-level Representation}  \label{sec:representation}
In this section, we look for explicit expressions of polynomial $P_k\in \{P_k^*, R_k, N_k\}$, for each member of the accelerated class. As will see, the value $P_k^2(\mu)$, taken at $\mu$, would be interpreted as the (spectral-level) convergence rate in the corresponding eigen-space.

\noindent{\bf Spectral-level decomposition.} 
Let $\{P_k\}_{k\in\mathbb{N}}$ be a sequence of real-valued polynomials. Suppose the error evolving under the algorithm obeys: $\xi_0=P_{k}(B)\xi_0$, with $\xi_0 = w_0-w^*$ being the initial error. Denote by $\{e_i\}$ the eigen-basis of $B$, i.e. $B e_i = \mu_i e_i, \forall i\in \mathcal{I}$, where $\mathcal{I}$ is the index set.

In terms of the eigen-basis, $\xi_k$ can be decomposed as $\xi_k=\sum_{i\in\mathcal{I}} \langle\xi_t,e_i\rangle e_i$. And since operator $P_k(B)$ commutes with $B$, each $e_i$ is also an eigen-vector of $P_k(B)$. Hence,
\begin{equation}
 \xi_k^{(i)}:= \langle\xi_k,e_i\rangle =  P_k(\mu_i)\langle\xi_0,e_i\rangle =P_k(\mu_i){\xi}^{(i)}_0.
 \label{eq:eigendecomp}
\end{equation}
Firstly, we see that eigen-components $\xi^{(i)}_k$ evolve independently from each other. Specifically, ${\xi}_{k}^{(i)}$ is determined by quantities only from its corresponding eigen-space: eigen-component of $\xi_0$, scalar value of $P_k$ at point $\mu_i$. Secondly, the spectral-level convergence rate in a particular eigen-space is measured by $P_k^2(\mu_i)$ solely, with smaller value implying faster convergence. These facts allow us to analyze the algorithms in each eigen-space independently. 

Based on these observations, we can reduce the problem of analyzing operator $P_k(B)$ to the one of analyzing the scalar-valued polynomial $P_k(\mu)$ on $sp(B)\subseteq [0,1]$ instead, which is a simpler problem.

\noindent{\bf Spectral-level convergence rates.} 
To analyze the polynomials $\{P_k\}_{k\in\mathbb{N}}$, we will derive their explicit expressions first.

Recall that we already have explicit expressions for power and Chebyshev methods, as in Eqs.~(\ref{eq:power}) and (\ref{eq:chebyshev}), respectively. But, to the best of our knowledge, explicit expressions for SOR and Nesterov's AGD  methods (corresponding to the recurrence relations Eqs.~(\ref{eq:richardsonrec}) and (\ref{eq:nesterovrec}) are not found in the literature. 
Below we derive the explicit expressions of polynomials $R_k$ and $N_k$.

For the purpose of simplifying expressions, we introduce the following notations: $\cosh \Theta = \mu/\rho$ and $\cosh\Psi = \sqrt{\mu/\rho}$, when $\mu\in(\rho,1]$; and $\cos \theta = \mu/\rho$ and $\cos\psi = \sqrt{\mu/\rho}$, when $\mu\in[0,\rho)$; and also $\cosh \Delta  = 1/\rho$ and $\cosh \Lambda  = \sqrt{1/\rho}$.
By utilizing the technique of solving linear difference equations, we can solve the recurrence relations, and have the following theorem:
\begin{thm}[Explicit expressions for $R_k$ and $N_k$]
If the algorithm obeys the recurrence relation in Eq.(\ref{eq:richardsonrec}) or (\ref{eq:nesterovrec}), 
 then $\xi_{k} = R_k(B)\xi_{0}$ or $\xi_{k} = N_k(B)\xi_{0}$, respectively, where $R_k(B)$ and $N_k(B)$  have the following analytic expressions on interval $[0,1]:$
 \begin{subequations}
 \begin{eqnarray}
 R_k(\mu) = \exp(-k\Delta)\times 
 \left\{ \begin{array}{ll}
\tanh \Delta \cot \theta \sin k\theta+ \cos k \theta, & \mu\in[0,\rho),\\
k\tanh\Delta +1, & \mu=\rho,\\
\tanh \Delta \coth \Theta \sinh k\Theta+ \cosh k \Theta, & \mu\in(\rho,1].
\end{array}
\right.\quad\ \label{eq:expressionsor}\\
 N_k(\mu) =  \mu^{k/2}\exp(-k\Lambda)\times 
 \left\{ \begin{array}{ll}
\tanh \Lambda \cot \psi \sin k\psi+ \cos k \psi, & \mu\in[0,\rho),\\
k\tanh\Lambda+1, & \mu=\rho,\\
\tanh \Lambda \coth \Psi \sinh k\Psi+ \cosh k \Psi, & \mu\in(\rho,1].
\end{array}
\right. \label{eq:expressionnesterov}
 \end{eqnarray}
 \label{eq:expressions}
 \end{subequations}
 \label{thm:poly}
\end{thm}
\begin{proof}
See proof in Appendix~\ref{sec:pfthm2}.
\end{proof}
\begin{remark} Although $R_k$ and $N_k$ are expressed in terms of trigonometric and hyperbolic functions and angles, one should be aware that they are polynomials of degree $k$.
\end{remark}
\begin{figure}
\centering
\includegraphics[width=8.4cm]{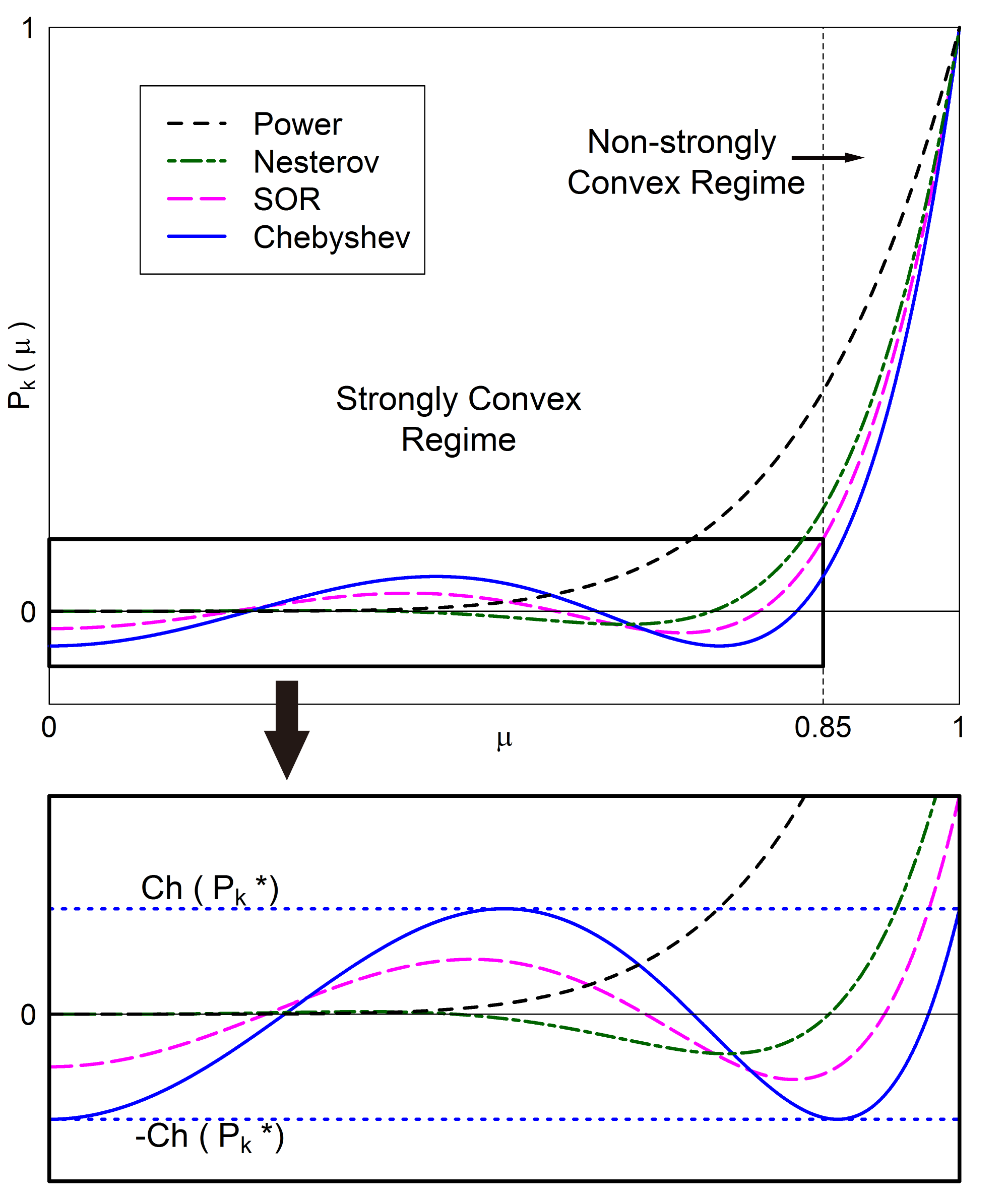}
\caption{Example plot for the polynomials $P_k\in \{P_k^*, R_k, N_k\}$, and $\mu^k$ for comparison, under the setting: $\rho=0.85, k=6$. The vertical dash line, corresponding to $\mu=\rho$ is for reference, and is the boundary of {\it strongly convex} and {\it non-strongly convex regimes}. Horizontal dashed lines in the lower plot show $\pm$ Chebyshev number of Chebyshev method. }\label{im:curves}
\end{figure}
Figure~\ref{im:curves} presents an example curve for each polynomial of the accelerated class and the power method (with $\rho=0.85, k=6$). It should be noted that left side of the figure, small $\mu$, corresponds to large Hessian eigenvalues, and right side, large $\mu$, corresponds to small Hessian eigenvalues.\\

\noindent{\bf Strongly convex and non-strongly convex regimes.} From Figure~\ref{im:curves}, one could observe the very distinct behaviours of the curves on the two sides of the vertical dashed line: on the left hand side, $P_k$ oscillate; but on the right hand side, $P_k$ are monotonically increasing.
Thus, we divide the spectrum space $[0,1)$ into two parts: the {\it strongly convex regime} (left side of vertical dash line in Figure~\ref{im:curves}), with $0\le\mu\le\rho$, corresponding to eigen-spaces that satisfy an $\alpha$-strongly convex condition; the {\it non-strongly convex regime} (left side of vertical dash line in Figure~\ref{im:curves}), with $\rho<\mu<1$, corresponding to eigen-spaces that break the $\alpha$-strongly convex condition. Note that this partition depends on our choice of parameters $\rho$.\\

Based on Figure~\ref{im:curves} and Theorem~\ref{thm:poly}, we observe that:
\begin{obser}\label{obs:1}
    Compared to the power method, polynomials of the accelerated class methods tend to take: larger  values $|P_k(\mu)|$ for small $\mu$'s (left  side of Figure~\ref{im:curves}); smaller values $|P_k(\mu)|$ for large $\mu$'s (right side of Figure~\ref{im:curves}).
\end{obser}
This observation indicates that these accelerated class methods converge slower than power method in eigen-spaces with very small $\mu$, or equivalently with large Hessian eigenvalues.
Then we immediately have the following important
\begin{remark}
The accelerated class methods do not accelerate convergence for all cases, specifically for cases in which Hessian eigenvalues are concentrated near the top of the spectrum. 
\end{remark}
However, they do accelerate in the worst case scenario, as is well-known in literature.
In Section~\ref{sec:strongly} we will see that they also provide acceleration in the non-strongly convex regime (for very small eigenvalues of the Hessian).\\

\noindent{\bf Reconstruction of excess risk.} With explicit expressions of such polynomials, we can reconstruct the excess risk $f(w_k)-f(w^*)$, once given spectral information, i.e. the value of $\mu_i$, $\forall i\in \mathcal{I}$, or distribution of $\mu$:
\begin{equation}
    {f(w_k)-f(w^*)} = \sum_{i\in\mathcal{I}} \beta(1-\mu_i)P_k^2(\mu_i) (\xi_0^{(i)})^2.
\end{equation}
Note that this excess risk is not directly computable, since $\xi_0$ will be never known. But if distribution of $\xi_0$ is somehow given or well-approximated, we can calculate an expected excess risk
\begin{equation}
\mathbb{E}_{} \left[{f(w_{k})-f(w^*)}\right] =  \sum_{i\in\mathcal{I}} \beta(1-\mu_i)P_k^2(\mu_i) \mathbb{E} [(\xi_{0}^{(i)})^2], \nonumber
\end{equation}
where the expectation is taken over distribution of $\xi_{0}$. For example, if we assume the initialization is isotropic, i.e. $\mathbb{E} [(\xi_{0}^{(i)})^2] = \mathbb{E} [(\xi_{0}^{(j)})^2], \forall i,j\in \mathcal{I}$, then the convergence rate of the expected excess risk can be computed by $\sum (1-\mu_i)P_k^2(\mu_i)/ \sum (1-\mu_i)$.

\section{Analysis in the Strongly Convex Regime}\label{sec:strongly}
In this section, we perform analysis on the accelerated class algorithms in the strongly convex regime, based on the explicit expressions for these algorithms.  Assuming $\alpha$-strong convexity with $\alpha = \beta/(1-\rho)$ would lead to the same results in this regime as  the eigen-components $\xi^{(i)}$ evolve independently from each other, and their behaviors only depend on the parameter $\rho$, which is determined by $\alpha$.\\

\noindent{\bf Worst case convergence rate.} Recall that the spectral-level convergence rate is solely determined by the value $P_k^2(\mu)$, and that smaller value implies faster convergence.  According to the definition of Chebyshev number, Eq.(\ref{eq:chebnumb}), and discussion of the Chebyshev method in Section~\ref{sec:algorithms},  we have the following claims:
\begin{claim}
 Square of {Chebyshev number}, $Ch_{\rho}^2$, measures the worst case convergence rate, under the strongly convex setting. Formally, $\forall \mu\in[0,\rho]$,
 \begin{equation}
     |P_k(\mu)|\le Ch_{\rho}(P_k), \ \forall P_k\in \{P_k^*, R_k, N_k\}, \forall k\ge 1.
 \end{equation}
\end{claim}
\begin{claim}[Optimality of the Chebyshev semi-iterative algorithm~\cite{golub1961chebyshev}]
 After the same number of iterations, Chebyshev algorithm achieves the lowest Chebyshev number among the accelerated class (and all possible first-order methods).
\end{claim}

In the following theorem, we show that the Chebyshev number is exactly the polynomial value taken at $\mu=\rho$, which is at the boundary of the regimes. This fact is also illustrated in Figure~\ref{im:curves}, for Chebyshev method.
\begin{thm}[Computation of the Chebyshev numbers]\label{thm:uppbound}
    \begin{equation}
           Ch_{\rho}(P_k) = P_k(\rho),\ \forall P_k\in \{P_k^*, R_k, N_k\}, \forall k\ge 1. \label{eq:thmstable}
    \end{equation}
    Moreover,
    \begin{subequations}
\begin{eqnarray}
   & & Ch_{\rho}(P_k^*) = 1/\cosh(k\Delta), \\
   & & Ch_{\rho}(R_k) = \exp(-k\Delta)(k\tanh\Delta +1), \\
   & & Ch_{\rho}(N_k) = {\rho}^{k/2}\exp(-k\Lambda)(k\tanh \Lambda +1).
\end{eqnarray}\label{ex}
\end{subequations}
\end{thm}
Eq.~(\ref{eq:thmstable}) of this theorem carries two important messages for the accelerated class: (a), $sp(B) = \{\rho\}$ is the worst case scenario; and (b), the worst case convergence rate can be exactly computed by the value $Ch_{\rho}^2(P_k) = P_k^2(\rho)$, where $P_k\in \{P_k^*, R_k, N_k\}$. Thus we have the explicit expressions of the Chebyshev numbers, shown in Eq.~(\ref{ex}).

\noindent{\bf Comparison of algorithms.} Based on the expressions of Chebyshev numbers, we compare the worst case convergence rates across algorithms, as shown below:
\begin{thm}[Worst-case comparison]\label{thm:comparebound}
The worst case convergence rates for the accelerated class algorithms satisfy: $\forall k\ge 1$,
\begin{equation}
    0 < Ch_{\rho}(P_k^*) < Ch_{\rho}(R_k) < Ch_{\rho}(N_k) < Ch_{\rho}(\mu^k)<1.
\end{equation}
\end{thm}
The above inequalities are  
 consistent with the optimality of the Chebyshev algorithm and the fact that the accelerated class algorithms converge faster than the power method in the worst case.
Moreover, we see that Nesterov's AGD has the slowest worst-case convergence rate, among the accelerated class algorithms.

Combining Theorems~\ref{thm:uppbound} and~\ref{thm:comparebound}, we get the following corollary which recovers the known convergence rates, which can be found in ~\cite{bubeck2015convex}:
\begin{corollary}\label{cor:2}
\begin{equation}
    Ch_{\rho}^2(P_k^*),\ Ch_{\rho}^2(R_k),\   Ch_{\rho}^2(N_k)  \sim O\left(\exp(-\frac{k}{\sqrt{\kappa}})\right).\nonumber
\end{equation}
\end{corollary}
Theorem~\ref{thm:comparebound} provides a qualitative comparison, to compare quantitatively, we look at the asymptotic case. We assume the (pseudo) condition number is sufficiently large, correspondingly $\rho$ is sufficiently close to $1$. Then we have:
\begin{thm}[Asymptotic Analysis] 
For $k \ge 1$ and small enough $1-\rho$, the Chebyshev numbers can be expressed as \label{thm:approrho}
\begin{subequations}
\begin{eqnarray}
  \textnormal{Power:}& &Ch_{\rho}(\mu^k) = 1-k(1-\rho) + o(1-\rho),\nonumber\\
  \textnormal{Chebyshev:}& &Ch_{\rho}(P_k^*)  = 1-k^2(1-\rho) + o(1-\rho),\nonumber\\
  \textnormal{SOR:}& &Ch_{\rho}(R_k) = 1-k^2(1-\rho) + o(1-\rho),\nonumber\\
  \textnormal{Nesterov's:}& &Ch_{\rho}(N_k) = 1-\frac{1}{2}(k^2+k)(1-\rho)+ o(1-\rho).\nonumber
\end{eqnarray}
\end{subequations}
\end{thm}
The coefficients, expressed in terms of $k$, of $1-\rho$ linear term indicate the asymptotic convergence rate. 
We observe that, for each the accelerated class algorithm, the coefficient of 1st-order term is quadratic in number of iterations $k$. This means faster convergence and is consistent with the fact that $T = O(\sqrt{\kappa}\log (1/\epsilon))$, as expected.\\

\noindent{\bf Exponential spectral-level convergence rate.} The following theorem states that each of the accelerated class algorithms converges exponentially in each eigen-space:
\begin{thm}[Exponential Convergence] \label{thm:expconv1}
Define $\tilde{\Delta} = \log\left(\frac{1+e^{2\Lambda}}{2}\right)$, then $0\le \tilde{\Delta}<\Delta$. And moreover, $\forall \mu\in[0,\rho]:$
    \begin{subequations}
    \begin{eqnarray}
    \forall \delta\in[0,\Delta),& &  \lim_{k\to \infty}e^{k\delta}P_k^*(\mu) =\lim_{k\to \infty} e^{k\delta}R_k(\mu)  =0;\nonumber\\
    \forall \delta\in[0,\tilde{\Delta}),& & \lim_{k\to \infty} e^{k\delta}N_k(\mu)=0.\nonumber
    \end{eqnarray}
    \end{subequations}
\end{thm}
These exponential spectral-level convergence rates are stronger than the results obtained in~\cite{neubauer2017nesterov}, in which a super-polynomial convergence rate is obtained.

\subsection{Discussion on Nesterov's AGD}

According to the polynomial expression of $N_k$, in Eq.(\ref{eq:expressionnesterov}), Nesterov's AGD seems to be a hybrid of power method and SOR. Specifically, the term $\mu^{k/2}$  corresponds to running $k/2$ iterations of power method, and the rest terms correspond to running $k$ more iterations of SOR, but on  a ''square rooted'' spectrum, i.e. $\mu \to \sqrt{\mu}$, $\rho\to \sqrt{\rho}$.

Noticing Observation~\ref{obs:1} and the appearance of the term $\mu^{k/2}$, it is reasonable to expect that Nesterov's AGD performs better than Chebyshev and SOR in eigen-spaces with larger Hessian-eigenvalue (correspondingly smaller $\mu)$, but performs worse in eigen-spaces with smaller Hessian-eigenvalue (correspondingly larger $\mu)$.\\

\noindent{\bf Slower worst case convergence rate.} Although Nesterov's AGD also have exponential convergence rates, as shown in Theorem~\ref{thm:expconv1}, the following theorem separates it from the other two accelerated class algorithms, by showing that it has a relatively slower worst case convergence rate.
\begin{thm}  \label{thm:nesterov}
Let  $\tilde{\Delta} = \log\left(\frac{1+e^{2\Lambda}}{2}\right)$ as in Theorem~\ref{thm:expconv1}, $\forall \delta$, s.t. $\tilde{\Delta}\le \delta <\Delta$,
\begin{equation}
    \lim_{k\to \infty} e^{k\delta}N_k(\rho)= \infty.
\end{equation}
\end{thm}
This fact is more explicitly illustrated in the asymptotic case, as shown in Theorem~\ref{thm:approrho}. The existence of $\frac{1}{2}$ before $k^2$ makes Nesterov's AGD has a relatively larger Chebyshev number, hence converges slower in the worst case scenario. 

Therefore, we conclude that Nesterov's AGD is not the optimal method in the sense of accelerating the worst-case scenario.

\section{Parametrized Accelerated Methods}\label{sec:parameterize}
As pointed out in the introduction, the assumption of bounded and known condition number $\kappa$ 
often does not hold in practice and can be problematic in both analysis and algorithm implementation:

Smooth kernel methods and neural networks are known to have very large or even unbounded condition numbers~\cite{2018arXiv180103437B,2016arXiv161107476S}. These condition numbers are generally difficult to estimate, since the estimation is prohibitively costly and numerically unstable. When the estimation is poor, there is  no theoretical guarantee for the validity of the accelerated class algorithms. Even if the condition number is known or well-estimated  but very large (e.g., $\sqrt{\kappa} \ll d)$, the exponential theoretical rate $O(e^{-k/\sqrt{\kappa}})$ can still be very slow, and potentially requires more computation than the Newton's method.

To address this issue, we propose to parametrize the accelerated class algorithms by treating $\rho$, or, equivalently, the ``condition number'' $\kappa$, as a free parameter. 

The parametrization  allows eigenvalues  to appear in the non-strongly convex regime, s.t.
$sp(B)\\ \not\subseteq [0,\rho]$. We  validate the parametrized accelerated class algorithms by showing that they also converge in the non-strongly convex regime, i.e. when $\mu\in(\rho,1)$. Moreover, we prove that these algorithms converge exponentially fast for each eigenvalue. Additionally, we show in the non-strongly convex regime  accelerated class methods converge uniformly faster than ordinary gradient descent (the 
power method).

\subsection{Performance in Non-strongly Convex Regime}
The validity of the accelerated class algorithms in non-strongly convex regime is guaranteed by the following convergence theorem:
\begin{thm}[Exponential convergence] \label{thm:nonstrongly}
Chebyshev, SOR, and Nesterov's AGD converge exponentially in every eigen-space in the non-convex regime, i.e. $\forall \mu\in(\rho,1):$
\begin{eqnarray}
   \forall \delta\in[0,\Delta-\Theta),\lim_{k\to \infty}e^{k\delta}P_k^*(\mu) =\lim_{k\to \infty} e^{k\delta}R_k(\mu)=0;& &\nonumber\\
   \forall \delta\in [0,\Lambda-\Psi),\lim_{k\to \infty}e^{k\delta}N_k(\mu) =0.\quad \quad \quad \quad \quad \quad \quad\ \ & & \nonumber
\end{eqnarray}
\end{thm}
\begin{remark} Since both $\Theta$ and $\Psi$ depend on $\mu$, the spectral-level convergence rates should also depend on $\mu$, with smaller $\mu$ (correspondingly larger Hessian-eigenvalue) having relatively faster convergence rate.
\end{remark}
Compare to the exponential spectral-level convergence in strongly convex regime, as in Theorem~\ref{thm:expconv1}, this exponential convergence is not uniform on this regime, since the range of valid $\delta$ shrinks to 0 as $\mu\to 1$.
\begin{figure}
    \centering
    \includegraphics[width=8cm]{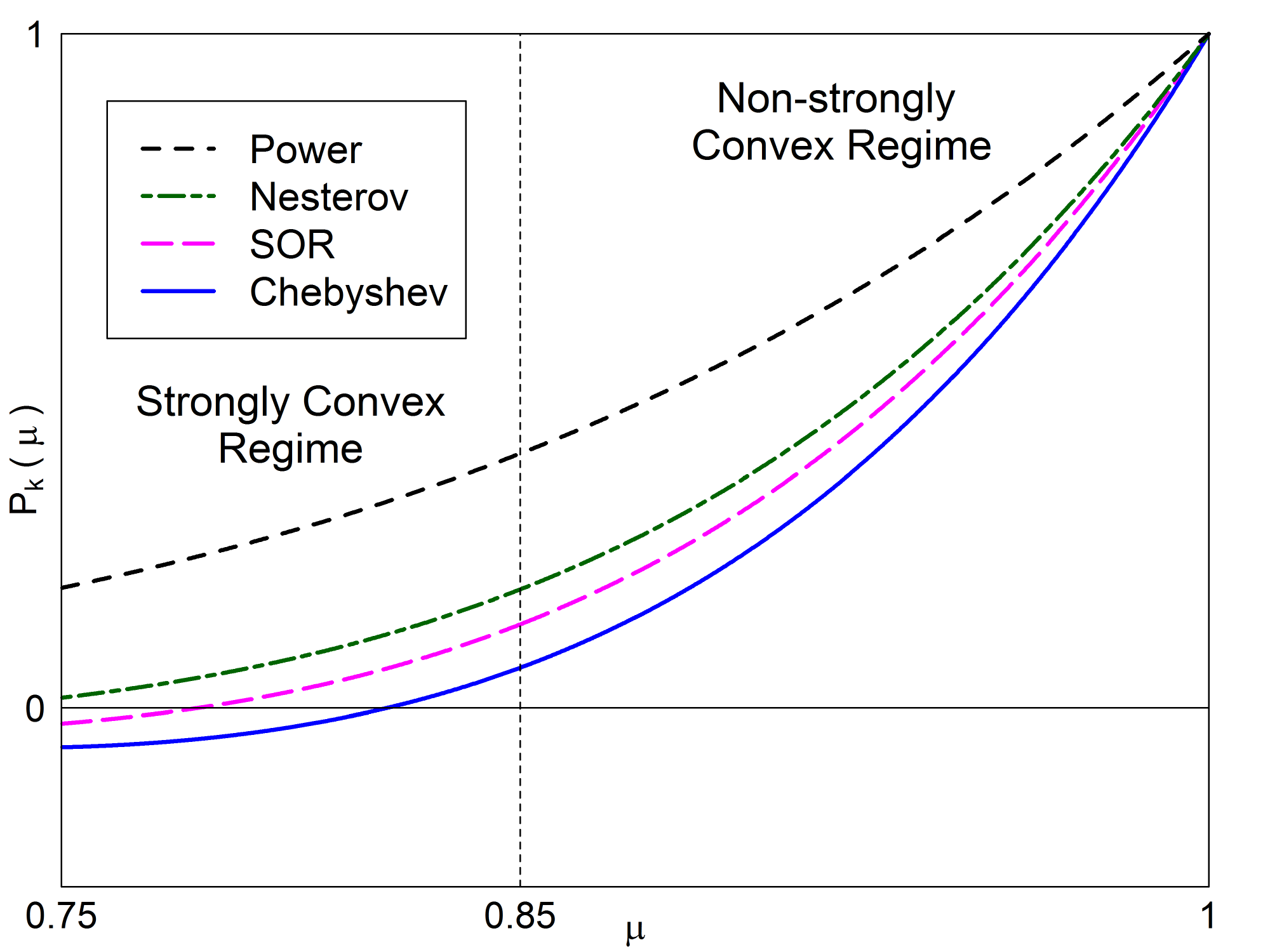}
    \caption{Example curves showing convergence behaviour in the non-strongly convex regime: $\rho = 0.85, k=6$, illustrating Theorem~\ref{thm:nonstronglycomp}. The vertical dash line separates the non-strongly convex and strongly convex regimes.}
    \label{fig:2}
\end{figure}\\

\noindent{\bf Comparison of algorithms.} We also compare the performance of these accelerated class algorithms in the non-strongly convex regime.
\begin{thm}[Comparison of algorithms]
In the non-strongly convex regime, i.e. $\forall \mu\in(\rho,1)$, we have
\begin{eqnarray}
    \textnormal{(a):} & & 0<P_k^*(\mu)< R_k(\mu)< \mu^k < 1,\nonumber\\
    \textnormal{(b):} & & 0<N_k(\mu)< \mu^k < 1.\nonumber
\end{eqnarray}\label{thm:nonstronglycomp}\end{thm}
\begin{remark}
Recall that $\mu^k$ is the polynomial expression of power method (ordinary gradient descent), which we list here for comparison.
\end{remark}

Part (a) of Theorem~\ref{thm:nonstronglycomp} gives an ordering of Chebyshev, SOR and power methods, in the non-strongly convex regime. Part (b) shows that Nesterov's AGD also converge faster than power method in this regime. From the theorem, we get the following message: in the non-strongly convex regime, the accelerated class algorithms always converge faster than power method (ordinary gradient descent).

Figure~\ref{fig:2} briefly illustrates the results of Theorem~\ref{thm:nonstronglycomp}. 

We currently do not have direct comparison of Nesterov's AGD with Chebyshev and SOR methods, but based on Theorem~\ref{thm:nonstrongly}, it is reasonable to conjecture that Nesterov's AGD, at least asymptotically, converges slower than the other two methods.

\subsection{Choosing Different Acceleration Parameters}
Noting that different choices of acceleration parameter result different polynomials, we use superscript $[i], i\in\{1,2\}$, to distinguish this difference.
\begin{thm}[Effect of choosing different parameters]\label{thm:choosepara}
Let  $0<\rho_1<\rho_2<1$, then $\forall  P_k \in \{P_k^*, R_k, N_k\}, \forall k>1:$
\begin{eqnarray}
 & &Ch_{\rho_1}(P_k^{[1]}) < Ch_{\rho_2}(P_k^{[2]}); \nonumber\\ 
 & &\forall \mu>\rho_2, \ P_k^{[1]}(\mu)>P_k^{[2]}(\mu).\nonumber
\end{eqnarray}
\end{thm}
\begin{figure}
    \centering
    \includegraphics[width=8cm]{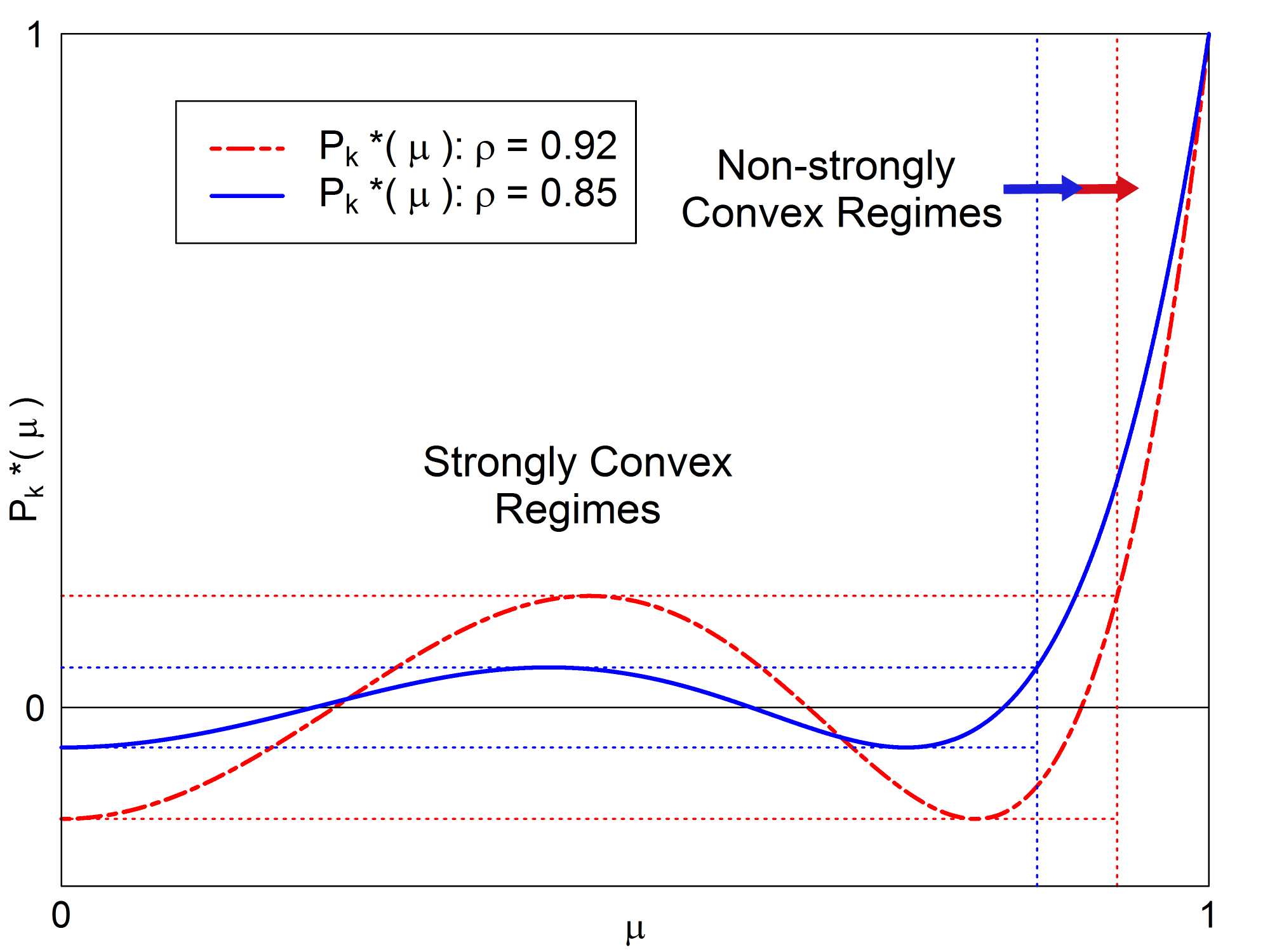}
    \caption{Illustration of choosing different acceleration parameters: plot shows curves of Chebyshev method, when setting $\rho=0.85$ (blue) and $\rho=0.92$ (red), respectively. Vertical dash lines indicate the position of boundaries of regimes, (i.e. $\mu=\rho$), and horizontal dash lines indicate $\pm$  of Chebyshev numbers.}
    \label{fig:my_label}
\end{figure}
Figure~\ref{fig:my_label} illustrates this theorem, see caption for details.

Loosely speaking, this theorem states that smaller $\rho$ tends to: (a) accelerate the convergence in strongly convex regime, $\mu\in [0,\rho_1]$, by lowering the corresponding Chebyshev number; and (b) slow down convergence in the non-strongly convex regime, $\mu\in(\rho_2,1)$. However, readers should be aware that changing parameter $\rho$ will also change the partition of the regimes. This effect is also shown in  Figure~\ref{fig:my_label}.

{\small
\bibliographystyle{plain}
\bibliography{main}
}

\clearpage

\appendix
\section{Appendix: Proof of Theorems}

\begin{lemma}\label{lemma:sin}
\begin{equation}\forall \theta \in [0,\pi/2],\quad |\sin k\theta| \le k\sin \theta,\quad k\in \mathbb{N}.\end{equation}
\end{lemma}
\begin{proof}
Obviously, the lemma hold for $n=1$.
In the following, we assume $n\ge 2$.

First consider the interval $[0,\frac{\pi}{2k}]$: 

$\forall \theta\in [0,\frac{\pi}{2k}]$, both $\sin k \theta$ and  $\sin k \theta$ are positive, and $\cos k\theta \le \cos \theta$, because of the monotonicity of $\cos \theta$ on $[0,\pi/2]$ and $0\le \theta\le k\theta\le \pi/2$. Since
\begin{eqnarray}
   \left(\sin k \theta\right)' = k\cos k\theta,\quad
    \left(k\sin \theta\right)' = k \cos \theta,
\end{eqnarray}
then $\left(\sin k \theta\right)'\le \left(k\sin \theta\right)'$. Combining with the fact that 
\begin{equation}
    \sin k \theta|_{\theta=0} = k\sin\theta|_{\theta=0} = 0,
\end{equation}
one can conclude that 
\begin{equation}
    \forall \theta \in [0,\frac{\pi}{2k}],\quad |\sin k\theta| \le k\sin \theta.\label{eqpf:lemma1}
\end{equation}
Then, we consider the interval $[\frac{\pi}{2k}, \frac{\pi}{2}]$:

 $\forall \theta\in[\frac{\pi}{2k}, \frac{\pi}{2}]$, we have
\begin{equation}
    |\sin k\theta| \le 1 = \sin k\frac{\pi}{2k} \le k \sin \frac{\pi}{2k} \le k \sin \theta.
\end{equation}
where we used Eq.(\ref{eqpf:lemma1}) for the second inequality and monotonicity of $\sin \theta$ on
$[\frac{\pi}{2k}, \frac{\pi}{2}]$ for the last inequality.

Hence, we conclude the lemma.
\end{proof}

\subsection{Proof of Theorem~{\ref{thm:poly}}}\label{sec:pfthm2}
\begin{proof}
To solve the recurrence relations, we follow the technique for solving linear difference equations. \\

{\bf Second order Richardson case.} The corresponding recurrence relation is Eq.(\ref{eq:richardsonrec}):
\begin{eqnarray}
  \xi_{k+1} = \gamma B \xi_k + (1-\gamma){\xi}_{k-1}, \quad k\ge 1;\quad  \xi_1 = B\xi_0.\nonumber
\end{eqnarray}
Now, we define auxiliary polynomials $Q_k(B)$ which satisfies
\begin{eqnarray}
  Q_{k+1}(B) = \gamma B Q_{k}(B) + (1-\gamma) Q_{k-1}(B), \ k\ge 1;\quad Q_1(B) = \gamma B; \quad Q_0(B) = I.\label{pfeq:1}
\end{eqnarray}
Note that, not like in Eq.(\ref{eq:richardsonrec}), we set $Q_1(B) = \gamma B$ instead of $B$. 

By induction, one can easily verify that $Q_k(B)$ is a polynomial in $B$ of degree $k$, and that
\begin{equation}
    \xi_{k} = Q_{k-1}(B)\xi_1 + (1-\gamma)Q_{k-2}(B)\xi_0, \ \xi \ge 2.\label{pfeq:4}
\end{equation}
Replace operator $B$ in Eq.(\ref{pfeq:1}) by scalar variable $x$, and then we utilize the standard technique for solving linear difference equations: consider $Q_k(x)$ as $k$-th power of $q(x)$, then
\begin{eqnarray}
q^{k+1}(x) = \gamma x q^{k}(x) + (1-\gamma) q^{k-1}(x), \ k\ge 1, \label{pfeq:2}
\end{eqnarray}
and 
\begin{equation}
    \quad q^1(x) = \gamma x; \quad q^0(x) = 1.\label{pfeq:3}
\end{equation}
Eq.(\ref{pfeq:2}) reduces to the following quadratic form
\begin{equation}
    q^2(x) = \gamma x q(x) + (1-\gamma),\label{pfeq:5}
\end{equation}
which has two roots $q_{\pm}(x)$. The general solution would be 
\begin{equation}
    Q_k(x) = c_1 q_+^k(x) + c_2 q_-^k(x),
\end{equation}
where coefficients $c_1$ and $c_2$ are determined by the initial condition Eq.(\ref{pfeq:3}).

For this particular case: when $0\le x< \rho$, the roots $q_{\pm}(x)$ are complex, and have the form $(\gamma-1)^{1/2}\exp(\pm i \theta)$, where $\cos\theta = x/\rho$; when $\rho \le x\le 1$,  $q_{\pm}(x)$ are real and have the form $(\gamma-1)^{1/2}\exp(\pm  \Theta)$, where $\cos\Theta = x/\rho$.

Then, after algebraic manipulations, we have
\begin{equation}
    Q_{k}(x) = (\gamma-1)^{k/2}\left\{\begin{array}{ll}
    \frac{\sin (k+1)\theta}{\sin \theta}  &  \textnormal{if } 0\le x< \rho, \\
    k+1                 &     \textnormal{if } x= \rho,   \\
    \frac{\sinh (k+1)\Theta}{\sinh \Theta}  &    \textnormal{if } \rho <x\le 1.
    \end{array}
    \right.
\end{equation}
Using Eq.(\ref{pfeq:4}) and noting that $\xi_k = R_k(B)\xi_0$, after some  algebraic manipulation, we have the expression for $R_k$ as shown in the Theorem.\\

{\bf Nesterov's AGD case.} The proof for the case of Nesterov's AGD is very analogous to that of second-order Richardson, so we omit some unnecessary steps. In this case, the auxiliary polynomials $Q_k(B)$ now satisfies
\begin{equation}
Q_{k+1}(B) = \gamma' B Q_{k}(B) + (1-\gamma')B Q_{k-1}(B), \ k\ge 1;\quad Q_1(B) = \gamma' B; \quad Q_0(B) = I.
\end{equation}
Please note the appearance of the additional $B$ in the term of $Q_{k-1}$, and the differently defined parameter $\gamma' = 2/(1+\sqrt{1-\rho})$.

Therefore, $q(x)$ now satisfies, instead of Eq.(\ref{pfeq:5}),
\begin{equation}
q^2(x) = \gamma' x q(x) + (1-\gamma')x.
\end{equation}
Then, $Q_k(x)$ is in turn
\begin{equation}
     Q_{k}(x) = (\gamma'-1)^{k/2}x^{k/2}\left\{\begin{array}{ll}
    \frac{\sin (k+1)\psi}{\sin \psi}  &  \textnormal{if } 0\le x< \rho, \\
    k+1                 &     \textnormal{if } x= \rho,   \\
    \frac{\sinh (k+1)\Psi}{\sinh \Psi}  &    \textnormal{if } \rho <x\le 1,
    \end{array}
    \right.
\end{equation}
where $\cos \psi = \sqrt{x/\rho}, x\in[0,\rho)$ and $\cosh \Psi = \sqrt{x/\rho}, x\in(\rho,1]$, as defined in Section~\ref{sec:representation}.

By induction, we can show that in this case
\begin{equation}
    \xi_{k} = Q_{k-1}(B)\xi_1 + (1-\gamma')B Q_{k-2}(B)\xi_0, \ k \ge 2.\label{pfeq:42}
\end{equation}
Noting that $\xi_k = N_k(B)\xi_0$, we can have the expression for $N_k$ as shown in the theorem.
\end{proof}

\subsection{Proof of Theorem~\ref{thm:uppbound}}
\begin{proof}
It is enough to show that,
\begin{equation}
    \forall \mu \in [0,\rho], |P_k(\mu)| \le P_k(\rho).
\end{equation}
Let's prove case by case:\\

\noindent {\bf Chebyshev.}  Note that $0\le \mu \le \rho$. According to Eq.(\ref{eq:chebpolynomial}) and (\ref{eq:chebyshev}),
\begin{equation}
    |P_k(\mu)|  = \frac{|\cos (k \cos^{-1}(\mu/\rho))|}{\cosh (k\cosh^{-1}(1/\rho))} \le \frac{1}{\cosh (k\cosh^{-1}(1/\rho))} = P_k(\rho).
\end{equation}

\noindent {\bf Second-order Richardson.} Since $0\le \mu \le \rho$ and $\theta = \cos^{-1}(\mu/\rho)$, then $\theta\in[0,\pi/2]$. According to Eq.(\ref{eq:expressionsor}),
\begin{eqnarray}
    |R_k(\mu)| &=&  \rho\exp (-k\Delta)\cdot |\sinh \Delta \cos \theta \frac{\sin k\theta}{\sin \theta}+ \cosh\Delta \cos k \theta| \nonumber\\
    &\le& \rho\exp (-k\Delta) \cdot \left(|\sinh \Delta \cos \theta \frac{\sin k\theta}{\sin \theta}| + |\cosh\Delta \cos k \theta|   \right) \nonumber\\
    &\le& \rho\exp (-k\Delta) \cdot \left( \sinh \Delta \frac{|\sin k\theta|}{\sin \theta} + \cosh\Delta       \right) \nonumber\\
    &\le& \rho\exp (-k\Delta) \cdot \left( k\sinh \Delta  + \cosh\Delta       \right), 
\end{eqnarray}
where the last inequality holds true because of Lemma~\ref{lemma:sin}. 

One the other hand, when $\mu=\rho$, the angle $\theta=0$, thus
\begin{equation}
    R_k(\rho) = \rho\exp (-k\Delta)\left( k\sinh \Delta  + \cosh\Delta       \right). 
\end{equation}
Combining the above two equations, we conclude the theorem for second-order Richardson case.\\

\noindent {\bf Nesterov's AGD.} This argument is similar to the second-order Richardson case. The angle $\psi = \cos^{-1}(\sqrt{\mu/\rho})$ is in the interval $[0,\pi/2]$. According to Eq.(\ref{eq:expressionnesterov}),
\begin{eqnarray}
    |N_k(\mu)| &=& \mu^{k/2}\sqrt{\rho}\exp(-k\Lambda)\cdot|\sinh \Lambda \cos \psi \sin k\psi/\sin \psi+ \cosh\Lambda \cos k \psi|\nonumber\\
    &\le& \mu^{k/2}\sqrt{\rho}\exp(-k\Lambda)\cdot\left( |\sinh \Lambda \cos \psi \sin k\psi/\sin \psi |+ |\cosh\Lambda \cos k \psi|         \right)\nonumber\\
    &\le& \mu^{k/2}\sqrt{\rho}\exp(-k\Lambda)\cdot\left( \sinh \Lambda  |\sin k\psi/\sin \psi |+ \cosh\Lambda         \right)\nonumber\\
    &\le& \rho^{k/2}\sqrt{\rho}\exp(-k\Lambda)\cdot\left( k\sinh \Lambda+ \cosh\Lambda         \right)\nonumber\\
    &=&N_k(\rho),
\end{eqnarray}
where we applied Lemma~\ref{lemma:sin} again in the last inequality.
\end{proof}

 \subsection{Proof of Theorem~{\ref{thm:comparebound}}}
 We show that $Ch_{\rho}(R_k) < Ch_{\rho}(N_k)$ here. 
 And the rest statements are noting else but a limiting case of Theorem~\ref{thm:nonstronglycomp}, when setting angles $\Theta,\Psi\to 0$. Please see proof of Theorem~\ref{thm:nonstronglycomp} in Section \ref{section:app10}.
 
 Consider function $g(x)$ defined as, for any $k\ge 2$,
 \begin{equation}
     g(x) = \cosh^k x e^{-kx}(k\tanh x+1), \  x\in (0,\infty).
 \end{equation}
 Function $g(x)$ is monotonically decreasing in its domain, because \begin{eqnarray}
     g'(x) &=& k\cosh^{k-1} x\sinh x \cdot e^{-kx}(k\tanh x + 1)\nonumber\\
     & & + \cosh^{k} x \cdot (-k) e^{-kx}(k\tanh x + 1)\nonumber\\
     & & + \cosh^k x e^{-kx}\cdot k\frac{1}{\cosh^2x}\nonumber\\
     &=& k\cosh^{k-2}x e^{-kx}\left(k\sinh^x + \sinh x\cosh x -  k\sinh x\cosh x - \cosh^2x +1\right)\nonumber\\
     &=& k(k-1)e^{-kx}\cosh^{k-2}x\sinh x (\sinh x - \cosh x)\nonumber\\
     &<& 0.
 \end{eqnarray}
 By definition of $\Delta$ and $\Lambda$, we can see that $\Lambda < \Delta$, hence
 \begin{equation}
     g(\Delta) < g(\Lambda).
 \end{equation}
 Namely,
 \begin{equation}
     \cosh^k \Delta e^{-k\Delta}(k\tanh \Delta+1) < \cosh^k \Lambda e^{-k\Lambda}(k\tanh \Lambda+1).
 \end{equation}
 Multiplying $\rho^k$ on both sides, and noting that $\cosh\Delta = 1/\rho$ and $\cosh\Lambda = 1/\sqrt{\rho}$, we conclude that 
 $Ch_{\rho}(R_k) < Ch_{\rho}(N_k)$.

 \subsection{Proof of Corollary~\ref{cor:2}}
 \begin{proof}
 By theorem~\ref{thm:comparebound}, it suffices to just prove $Ch_{\rho}(N_k) \sim O(\exp(-k/\sqrt{\kappa}))$.
 
 Since $\cosh\Lambda = 1/\sqrt{\rho}$ and $\rho = 1-1/\kappa$, then
 \begin{eqnarray}
 \rho^{k/2}e^{-k\Lambda} = \rho^{k/2}(\cosh \Lambda - \sinh \Lambda)^k
 = \rho^{k/2}\left(\frac{1-\sqrt{1-\rho}}{\sqrt{\rho}}\right)^k = \left(1-\frac{1}{\sqrt{\kappa}}\right)^k.
 \end{eqnarray}
 Noting that $\tanh\Lambda <1$, we have
 \begin{equation}
     Ch_{\rho}(N_k) < \left(1-\frac{1}{\sqrt{\kappa}}\right)^k(k+1).
 \end{equation}
 Therefore, 
 \begin{equation}
      Ch_{\rho}^2(N_k) < \left(1-\frac{1}{\sqrt{\kappa}}\right)^k\cdot (k+1)^2\left(1-\frac{1}{\sqrt{\kappa}}\right)^k.
 \end{equation}
 We note that the first term $(1-1/\sqrt{\kappa})^k \sim O(\exp(-k/\sqrt{\kappa}))$, and the rest $(k+1)^2\left(1-\frac{1}{\sqrt{\kappa}}\right)^k\sim O(1)$. 
 
 Hence we conclude.
 \end{proof}

\subsection{Proof of Theorem~\ref{thm:approrho}}
\begin{proof}
Within the scope of this proof, we denote $\epsilon:= 1-\rho$, for the sake of simplicity. Outside of the scope,  $\epsilon$ could bear other meanings. 

As we know $\cosh \Delta = 1/\rho$, then
\begin{eqnarray}
    e^{\Delta} &=& \frac{1}{\rho}(1+\sqrt{1-\rho^2}) \nonumber\\
    &=&\frac{1}{1-\epsilon}(1+\sqrt{2\epsilon}\sqrt{1-\epsilon/2})\nonumber\\
    &=& (1+\epsilon+ O(\epsilon^2))(1+\sqrt{2\epsilon}+O(\epsilon^{3/2}))\nonumber\\
    &=& 1+\sqrt{2\epsilon}+\epsilon + O(\epsilon^{3/2}).
\end{eqnarray}
Similarly,
\begin{equation}
    e^{-\Delta} = 1-\sqrt{2\epsilon}+\epsilon + O(\epsilon^{3/2}).
\end{equation}
Therefore, for $k\in\mathbb{N}$,
\begin{eqnarray}
    e^{k\Delta} &=& (1+\sqrt{2\epsilon}+\epsilon + O(\epsilon^{3/2}))^k\nonumber\\
    &=& 1+k\sqrt{2\epsilon} + k\epsilon + \frac{k(k-1)}{2}2\epsilon + O(\epsilon^{3/2})\nonumber\\
    &=& 1+k\sqrt{2\epsilon} + k^2\epsilon +O(\epsilon^{3/2}).
\end{eqnarray}
And similarly,
\begin{equation}
     e^{-k\Delta} = 1-k\sqrt{2\epsilon} + k^2\epsilon +O(\epsilon^{3/2}).
\end{equation}

Hence, according to Theorem~\ref{thm:uppbound},
\begin{eqnarray}
    Ch_{\rho}(P_k^*) = \frac{1}{\cosh(k\Delta)}
    = \frac{2}{e^{k\Delta}+e^{-k\Delta}}
    =\frac{1}{1+k^2\epsilon+o(\epsilon)}
    = 1-k^2\epsilon + o(\epsilon);
\end{eqnarray}
and \begin{eqnarray}
    Ch_{\rho}(R_k) &=& e^{-k\Delta}(k\tanh\Delta+1)\nonumber\\
    &=& e^{-k\Delta}(k\frac{e^{\Delta}-e^{-\Delta}}{e^{\Delta}+e^{-\Delta}}+1)\nonumber\\
    &=& \left(1-k\sqrt{2\epsilon} + k^2\epsilon +o(\epsilon)\right)\left(1+k\frac{\sqrt{2\epsilon}}{1+\epsilon}+o(\epsilon)\right)\nonumber\\
    &=& 1-k^2\epsilon + o(\epsilon).
\end{eqnarray}
\\

Since $\cosh \Lambda = \sqrt{1/\rho}$, then
\begin{eqnarray}
    e^{\Lambda} = \frac{1}{\sqrt{\rho}}(1+\sqrt{1-\rho})=
   \frac{1}{\sqrt{1-\epsilon}}(1+\sqrt{\epsilon}) = 1+\sqrt{\epsilon} +\epsilon/2 + o(\epsilon).
\end{eqnarray}
Similarly,
\begin{equation}
    e^{-\Lambda} = 1-\sqrt{\epsilon} +\epsilon/2 +o(\epsilon).
\end{equation}
Therefore, for $k\in \mathbb{N}$,
\begin{eqnarray}
    e^{k\Lambda} &=& (1+\sqrt{\epsilon} +\epsilon/2 + o(\epsilon))^k = 1+k\sqrt{\epsilon} + \frac{k^2}{2}\epsilon +o(\epsilon),\\
    e^{-k\Lambda} &=& (1-\sqrt{\epsilon} +\epsilon/2 + o(\epsilon))^k = 1-k\sqrt{\epsilon} + \frac{k^2}{2}\epsilon +o(\epsilon).
\end{eqnarray}

According to Theorem~\ref{thm:uppbound},
\begin{eqnarray}
    Ch_{\rho}(N_k) &=& \rho^{k/2}e^{-k\Lambda}(1+k\tanh \Lambda)\nonumber\\
    &=& (1-\epsilon)^{k/2}(1-k\sqrt{\epsilon} + \frac{k^2}{2}\epsilon +o(\epsilon))(1+k\frac{\sqrt{\epsilon}}{1+\epsilon/2}+o(\epsilon))\nonumber\\
    &=& 1-\frac{1}{2}(k^2+k)\epsilon +o(\epsilon).
\end{eqnarray}
\end{proof}

\subsection{Proof of Theorem~{\ref{thm:expconv1}}}
\begin{proof}
To prove $0<\tilde{\Delta}<\Delta$, it suffices to prove $1<(1+\exp(2\Lambda))/2<\exp\Delta$. The first inequality is easy to see, after noticed that $\Lambda>0$.

Since $\exp\Lambda = \cosh \Lambda + \sinh\Lambda$ and $\cosh \Lambda = \sqrt{1/\rho}$, we have
\begin{equation}
e^{\Lambda} = \frac{1}{\sqrt{\rho}}\left(1+\sqrt{1-\rho}\right).
\end{equation}
Then
\begin{equation}
    2/(1+\exp(2\Lambda)) = \frac{\rho}{1-\sqrt{1-\rho}}.
    \label{eq:pfthm21}
\end{equation}
Meanwhile, 
\begin{equation}
    e^{\Delta} = \cosh \Delta + \sinh\Delta  = \frac{1+\sqrt{1-\rho^2}}{\rho}.
    \label{eq:pfthm22}
\end{equation}
For $0<\rho<1$, 
\begin{equation}
    \rho^2 = (1+\sqrt{1-\rho^2})(1-\sqrt{1-\rho^2})<(1+\sqrt{1-\rho^2})(1-\sqrt{1-\rho}).\label{pfeq:3-1}
\end{equation}
Combining Eq.(\ref{pfeq:3-1}) with Eq.(\ref{eq:pfthm21}) and (\ref{eq:pfthm22}), we conclude that $(1+\exp(2\Lambda))/2<\exp\Delta$.

For a given $\rho$, the corresponding Chebyshev numbers for the mentioned algorithms are 
\begin{subequations}
\begin{eqnarray}
& &Ch_{\rho}(P_k^*) =  P_k^*(\rho) = \frac{1}{\cosh (k \cosh^{-1}(1/\rho))} = \frac{2}{e^{k\Delta}+e^{-k\Delta}};\\
& &Ch_{\rho}(R_k) =R_k(\rho) = \rho e^{-k\Delta}(k\sinh\Delta + \cosh\Delta) =  e^{-k\Delta}\left(k\frac{\sinh\Delta}{\cosh\Delta}+1\right);\\
& &Ch_{\rho}(N_k) =N_k(\rho) = \left(\sqrt{\rho}e^{-\Lambda}\right)^t\left(k \coth\Lambda +1\right) = e^{-k\tilde{\Delta}}\left(k \coth\Lambda +1\right). \label{pfeq:8-1}
\end{eqnarray}
\end{subequations}
Then, for $0\le \delta<\Delta$,
\begin{subequations}
\begin{eqnarray}
    & &e^{k\delta}Ch_{\rho}(P_k^*) = \frac{2e^{k\delta}}{e^{k\Delta}+e^{-k\Delta}} = \frac{2}{e^{k(\Delta-\delta)}+e^{-k(\Delta+\delta)}} \to 0, \ \textrm{as } k\to\infty;\\
    & &e^{k\delta}Ch_{\rho}(R_k) = e^{-k(\Delta-\delta)}\left(k\frac{\sinh\Delta}{\cosh\Delta}+1\right)\to 0, \  \textrm{as } k\to \infty.
\end{eqnarray}
\end{subequations}
And for $0\le \delta<\tilde{\Delta}$,
\begin{equation}
    e^{k\delta}Ch_{\rho}(N_k) = e^{-k(\tilde{\Delta}-\delta)}\left(k \coth\Lambda +1\right)\to 0, \  \textrm{as } k\to \infty.
\end{equation}
Then one can conclude the theorem, after noting the fact that $|P_k(\mu)| \le Ch_{\rho}(P_k), \forall \mu\in[0,\rho]$, where $P_k \in \{P_k^*, R_k, N_k \}$.
\end{proof}

\subsection{Proof of Theorem~\ref{thm:nesterov}}
\begin{proof}
As for the Chebyshev number for Nesterov's AGD, we use the expression in Eq.(\ref{pfeq:8-1}) again. Then, for $\tilde{\Delta}\le \delta <\Delta$,
\begin{equation}
     e^{k\delta}Ch_{\rho}(N_k) = e^{-k(\tilde{\Delta}-\delta)}\left(k \coth\Lambda +1\right).
\end{equation}
Clearly, the exponent is non-negative, then it blows up, hence does not converge to 0.
\end{proof}

\subsection{Proof of Theorem~\ref{thm:nonstrongly}}
\begin{proof}
For $\mu\in [\rho,1]$,
\begin{subequations}
\begin{eqnarray}
& &P_k^*(\mu) = \frac{\cosh (k \cosh^{-1}(\mu/\rho))}{\cosh (k \cosh^{-1}(1/\rho))} = \frac{e^{k\Theta}+e^{-k\Theta}}{e^{k\Delta}+e^{-k\Delta}};\\
& &R_k(\mu) =  e^{-k\Delta}\left(\frac{\sinh \Delta}{\cosh \Delta}\frac{\cosh \Theta}{\sinh\Theta}\sinh k\Theta + \cosh k\Theta\right) \le  e^{-k\Delta}\left(\sinh k\Theta + \cosh k\Theta\right) = e^{-k(\Delta-\Theta)};\\
& &N_k(\mu) = (\sqrt{\mu}e^{-\Lambda})^k\left(\frac{\sinh \Lambda}{\cosh\Lambda}\frac{\cosh \Psi}{\sinh\Psi}\sinh k\Psi + \cosh k\Psi\right)\le (\sqrt{\mu}e^{-\Lambda})^k\left(\sinh k\Psi + \cosh k\Psi\right).
\end{eqnarray}
\end{subequations}
Then for $0\le \delta<\Delta-\Theta$,
\begin{subequations}
\begin{eqnarray}
& &e^{k\delta}P_k^*(\mu)=\frac{e^{k(\Theta+\delta)}+e^{-k(\Theta-\delta)}}{e^{k\Delta}+e^{-k\Delta}}\le \frac{2}{e^{k(\Delta-\Theta-\delta)}+e^{-k(\Delta+\Theta+\delta)}}\to 0, \ \textrm{as } k\to \infty;\\
& &e^{k\delta}R_k(\mu) \le e^{-k(\Delta-\Theta-\delta)}\to 0,  \ \textrm{as } k\to \infty.
\end{eqnarray}
\end{subequations}
And for $0\le \delta<\Lambda-\Psi$,
\begin{equation}
    e^{k\delta}N_k(\mu) =  e^{k\delta}(\sqrt{\mu}e^{-\Lambda})^ke^{k\Psi}\le e^{-k(\Lambda-\Psi-\delta)}\to 0, \ \textrm{as } k\to \infty.
\end{equation}
\end{proof}

\subsection{Proof of Theorem~\ref{thm:nonstronglycomp}}\label{section:app10}
\begin{proof}
{\it{ Step 1.} Prove $P_k^*(\mu) < R_k(\mu), \forall \mu \in (\rho,1)$.}

To compare the two polynomials on the interval $(\rho,1)$, we simply subtract one from the other, and then look at the positiveness. According to Eq.(\ref{eq:chebyshev}) and (\ref{eq:expressionsor}),
\begin{eqnarray}
    R_k(\mu)-P_k^*(\mu) &=&e^{-k\Delta}\left(\frac{\sinh \Delta}{\cosh \Delta}\frac{\cosh \Theta}{\sinh\Theta}\sinh k\Theta + \cosh k\Theta\right) - \frac{\cosh k\Theta}{\cosh k\Delta} \nonumber\\
    &=& \frac{1}{e^{k\Delta}\sinh \Theta \cosh \Delta \cosh k\Delta}\times \left[\sinh \Delta\cosh \Theta \sinh k\Theta \cosh k\Delta \right. \nonumber\\
    & &\quad - \left. \sinh \Theta \cosh \Delta \sinh k\Delta\cosh k\Theta \right] \nonumber\\
    &=& \frac{1}{e^{k\Delta}\sinh \Theta \cosh \Delta \cosh k\Delta}\times\frac{1}{2}\sinh(k-1)\Theta\sinh(k-1)\Delta\nonumber\\
    & & \quad \times \left[\frac{\sinh(k+1)\Delta}{\sinh(k-1)\Delta}-\frac{\sinh(k+1)\Theta}{\sinh(k-1)\Theta}\right]
\end{eqnarray}
Since the common factor is always positive, the positiveness of $R_k(\mu)-P_k^*(\mu)$ is determined by the square-bracketed stuff. Define auxiliary function 
\begin{equation}
    h(x):=\frac{\sinh(k+1)x}{\sinh(k-1)x}, \quad x>0. \label{eqpf: auxfun}
\end{equation}
Since $\Delta>\Theta$, it would be enough to show that $h(x)$ is strictly monotonically increasing, or equivalently $h'(x)>0$, when $x>0$.

Take derivative of $h(x)$,
\begin{equation}
h'(x) = \frac{1}{\sinh^2(k-1)x}(\sinh 2kx - k\sinh 2x), \quad x>0.
\label{deriv}
\end{equation}
On the other hand, noticing
\begin{equation}
(\sinh 2kx - k\sinh 2x)' = 2k(\cosh 2kx-\cosh 2x)>0, \textrm{ when } x>0 \wedge k>1,\nonumber
\end{equation}
and \begin{equation}(\sinh 2kx - k\sinh 2x)|_{x=0} = 0,\end{equation} we can see Eq.(\ref{deriv}) is always positive for $x>0$.

Therefore, we finish step 1.\\

{\it { Step 2.} Prove $R_k(\mu) < \mu^k, \forall \mu \in (\rho,1)$.}

Consider the following two functions:
\begin{equation}
    g_1(\theta) = \frac{\cosh k\theta}{\cosh^k\theta}; \quad g_2(\theta) = \frac{\cosh\theta\sinh k\theta}{\sinh \theta \cosh^k\theta}; \quad \theta\in (0,\infty), k\in \mathbb{N}.
\end{equation}
Their derivatives are
\begin{equation}
    g_1'(\theta) = \frac{k}{\cosh^{k+1}\theta}\sinh (k-1)\theta,
\end{equation}
and 
\begin{equation}
    g_2'(\theta) = \frac{1}{\sinh^2\theta\cosh^k\theta}\left[k\sinh\theta \cosh (k-1)\theta -\sinh k\theta\right].
\end{equation}
It is not hard to check that both $g_1'$ and $g_2'$ are positive for all $\theta\in (0,\infty)$ and $k >2$. Hence $g_1$ and $g_2$ are both monotonically increasing.

For $\mu \in (\rho,1)$, since $\cosh \Theta = \mu/\rho$ and $\cosh\Delta = 1/\rho$, then $0<\Theta <\Delta$. Thus,
\begin{equation}
    \frac{\cosh \Delta}{\sinh \Delta}g_1(\Theta)+g_2(\Theta) < \frac{\cosh \Delta}{\sinh \Delta}g_1(\Delta)+g_2(\Delta).
\end{equation}
Plug in expressions of $g_1$ and $g_2$, then multiply on both sides the factor $\tanh \Delta\cosh^{k}\Theta \exp(-k\Delta)$, one get
\begin{equation}\label{eqpf:3}
    \rho e^{-k\Delta}\left(\coth\Theta \sinh\Delta \sinh k\Theta + \cosh \Delta \cosh k\Theta \right) < \mu^k.
\end{equation}
\\

{\it { Step 3.} Prove $N_k(\mu) < \mu^k, \forall \mu \in (\rho,1)$.}

Make the following replacement in Eq.(\ref{eqpf:3}):
\begin{equation}
    \mu \to \sqrt{\mu}, \ \rho\to \sqrt{\rho},
\end{equation}
and multiply $\mu^{k/2}$ on both sides, we finish the step 3.
\end{proof}

\subsection{Proof of Theorem~\ref{thm:choosepara}}
\begin{proof}
{\it Part 1. } 

Consider the Chebyshev number $Ch_{\rho}(P_k)$ as a function of $\rho$, and then it is sufficient to show that $Ch_{\rho}(P_k)$ is monotonically increasing, i.e. $d Ch_{\rho}(P_k)/d\rho$ is always non-negative, on the interval $\rho\in(0,1)$.

By Theorem~\ref{thm:uppbound}, we have, for $0<\rho<1$,
\begin{subequations}
\begin{eqnarray}
    \frac{dCh_{\rho}(P_k^*)}{d\rho} &=& \frac{d Ch_{\rho}(P_k^*)}{d\Delta}\frac{d\Delta}{d\rho} = -\frac{k\sinh k\Delta}{\cosh^2 k\Delta}\frac{-1}{\rho\sqrt{1-\rho^2}} \ge 0, \\
    \frac{dCh_{\rho}(R_k)}{d\rho} &=& \frac{d Ch_{\rho}(R_k)}{d\Delta}\frac{d\Delta}{d\rho} \nonumber\\
    &=& \left[-k e^{-k\Delta}(k\tanh \Delta +1) + ke^{-k\Delta}\frac{1}{\cosh^2\Delta}  \right]\frac{-1}{\rho\sqrt{1-\rho^2}}\nonumber\\
    &=& ke^{-k\Delta}\frac{\sinh\Delta}{\cosh^2\Delta} (k\cosh\Delta+\sinh\Delta)\frac{1}{\rho\sqrt{1-\rho^2}}\nonumber\\
    &\ge& 0.
\end{eqnarray}
\end{subequations}
As for $Ch_{\rho}(N_k)$, we first note that $\exp(-k\Lambda)(k\tanh\Lambda+1)$ is positive and monotonically increasing on $\rho\in(0,1)$, because $\Lambda\ge 0$ and 
\begin{eqnarray}
    \frac{d}{d\rho}[\exp(-k\Lambda)(k\tanh\Lambda+1)] &=&\frac{d}{d\Lambda}[\exp(-k\Lambda)(k\tanh\Lambda+1)]\frac{d\Lambda}{d\rho}\nonumber\\
    &=& ke^{-k\Lambda}\frac{\sinh\Lambda}{\cosh^2\Lambda} (k\cosh\Lambda+\sinh\Lambda)\frac{1}{2\rho\sqrt{1-\rho}}
    \nonumber\\
    &\ge& 0.
\end{eqnarray}
And $Ch_{\rho}(N_k)$, as a function of $\rho$, is the product of function $\exp(-k\Lambda)(k\tanh\Lambda+1)$ and function $\rho^{k/2}$, which is also positive and  monotonically increasing on $\rho\in(0,1)$. Therefore, $Ch_{\rho}(N_k)$ is monotonically increasing on $\rho\in(0,1)$.\\

{\it Part 2. } 

The strategy of proof is similar to that in Part 1: for any fix $\mu\in(0,1)$, we prove that $P_k(\mu)$, as a function of $\rho$, is monotonically {\it decreasing} on $\rho\in(0,\mu)$. (Here, We need $\rho< \mu$ to make sure that $\mu$ is in the non-strongly convex regime.)

Omitting some tedious calculation steps, we have, according to Eq.(\ref{eq:chebyshev}),
\begin{eqnarray}
    \frac{d P_k^*(\mu)}{d\rho} &=& \frac{d }{d\rho}\left(\frac{\cosh k\Theta}{\cosh k\Delta}\right)\nonumber\\
    &=& \frac{k}{2} \frac{\cosh \Delta \sinh (k-1)\Delta\sinh (k-1)\Theta}{\cosh^2 k\Delta\sinh \Delta\sinh \Theta}\left[ \frac{\sinh(k+1)\Theta}{\sinh (k-1)\Theta} - \frac{\sinh(k+1)\Delta}{\sinh (k-1)\Delta}\right]
\end{eqnarray}
With the assistance from the auxiliary function $h$ defined in Eq.(\ref{eqpf: auxfun}), we know that the stuff in the squared bracket is negative, since $h$ is monotonically decreasing and $0<\Theta<\Delta$. 

Hence $d P_k^*(\mu)/d\rho<0$, then $P_k^*(\mu)$, as a function in $\rho$, is monotonically decreasing on $\rho\in(0,\mu)$.

Similarly, according to Eq.(\ref{eq:expressionsor}), we have,
\begin{eqnarray}
    \frac{d R_k(\mu)}{d\rho} &=& \frac{d }{d\rho}\left[e^{-k\Delta}(\tanh \Delta \coth \Theta \sinh k\Theta + \cosh k\Theta)\right]\nonumber\\
    &=& e^{-k\Delta}\sinh \Delta\cosh k\Theta \left[k(\coth^2\Delta-\coth^2\Theta) + \tanh k\Theta \coth \Theta(\frac{1}{\sinh^2\Theta} - \frac{1}{\sinh^2\Delta})    \right]\nonumber
\end{eqnarray}
We note the facts that $1/\sinh^2\Theta>1/\sinh^2\Delta$, and that $\tanh k\Theta< k\tanh\Theta$ for $0<\Theta<\Delta$, then following the above formula, we have
\begin{eqnarray}
    \frac{d R_k(\mu)}{d\rho} &<& e^{-k\Delta}\sinh \Delta\cosh k\Theta \left[k(\coth^2\Delta-\coth^2\Theta) + k(\frac{1}{\sinh^2\Theta} - \frac{1}{\sinh^2\Delta})    \right]=0,
\end{eqnarray}
because $\coth^2 x - 1/\sinh^2x = 1$.

Hence, $R_k(\mu)$, as a function of $\rho$, is monotonically {decreasing} on $\rho\in(0,\mu)$.

The argument about $N_k(\mu)$ directly follows the result for $R_k(\mu)$: the term $\mu^{t/2}$, in Eq.(\ref{eq:expressionnesterov}), is a constant (independent of $\rho$), and the remaining part is exactly same as $R_k$, after the replacements: $\mu\to\sqrt{\mu}, \rho\to \sqrt{\rho}$.

\end{proof}

\end{document}